\documentclass[twoside]{article} 
\usepackage[accepted]{aistats2012}
\usepackage[round]{natbib}
\usepackage{times}
\usepackage{float}
\usepackage{amsthm}
\usepackage{amsmath}
\usepackage{amsfonts}
\usepackage{mathtools}
\usepackage{hyperref}

\makeatletter

\floatstyle{ruled}
\newfloat{algorithm}{tbp}{loa}
\floatname{algorithm}{Algorithm}

\numberwithin{equation}{section}
\numberwithin{figure}{section}

\newtheorem{theorem}{Theorem}[section]

\newtheorem{lemma}{Lemma}[section]
\newtheorem{corollary}{Corollary}[section]

\newtheorem{assumption}{Assumption}

\makeatother

\newcommand{\set}[1]{\{#1\}}
\newcommand{\Set}[1]{\left\{#1\right\}}
\newcommand{\Bracks}[1]{\left[#1\right]}
\newcommand{\BigBracks}[1]{\Bigl[#1\Bigr]}
\newcommand{\BiggBracks}[1]{\Biggl[#1\Biggr]}
\newcommand{\bigBracks}[1]{\bigl[#1\bigr]}
\newcommand{\Parens}[1]{\left(#1\right)}
\newcommand{\bigParens}[1]{\bigl(#1\bigr)}
\newcommand{\BigParens}[1]{\Bigl(#1\Bigr)}
\newcommand{\Card}[1]{\left\lvert#1\right\rvert}

\DeclareMathOperator*{\E}{\mathop{\mathbf{E}}}
\newcommand{\Et}{\mathop{\mathbf{E}_t}}
\newcommand{\Etp}{\mathop{\mathbf{E}_{t'}}}
\newcommand{\var}{\mathop{\mathbf{Var}}}
\newcommand{\vart}{\mathop{\mathbf{Var}_t}}
\newcommand{\vartp}{\mathop{\mathbf{Var}_{t'}}}
\renewcommand{\Pr}{\mathbf{P}}
\newcommand{\R}{\mathbb{R}}
\newcommand{\Rhat}{\hat{R}}

\newcommand{\X}{\ensuremath{\mathcal{X}}}
\newcommand{\A}{\ensuremath{\mathcal{A}}}

\newcommand{\order}{\ensuremath{\mathcal{O}}}

\newcommand{\regret}{\text{\itshape regret}}
\newcommand{\algo}{\textsf{\upshape A}}

\def\atil{\tilde{a}}
\def\xul{\underline{x}}
\def\yul{\underline{y}}
\def\aul{\underline{a}}
\def\bul{\underline{b}}

\DeclareMathOperator*{\argmax}{argmax}

\begin{document}

\runningauthor{Agarwal, Dud\'ik, Kale, Langford and Schapire}
\twocolumn[
\aistatstitle{Contextual Bandit Learning with Predictable Rewards}

\aistatsauthor{Alekh Agarwal\\\texttt{\small alekh@cs.berkeley.edu} \And
  Miroslav Dud\'ik\\\texttt{\small mdudik@yahoo-inc.com} \And Satyen
  Kale\\\texttt{\small sckale@us.ibm.com} \AND John
  Langford\\\texttt{\small jl@yahoo-inc.com} \And Robert
  E. Schapire\\\texttt{\small schapire@cs.princeton.edu}} 

\aistatsaddress{}
]

\begin{abstract}
  Contextual bandit learning is a reinforcement learning problem where
  the learner repeatedly receives a set of features (context), takes
  an action and receives a reward based on the action and context. We
  consider this problem under a realizability assumption: there exists
  a function in a (known) function class, always capable of predicting
  the expected reward, given the action and context.  Under this
  assumption, we show three things. We present a new
  algorithm---Regressor Elimination--- with a regret similar to the
  agnostic setting (i.e. in the absence of realizability
  assumption). We prove a new lower bound showing no algorithm can
  achieve superior performance in the worst case even with the
  realizability assumption. However, we do show that for \emph{any}
  set of policies (mapping contexts to actions), there is a
  distribution over rewards (given context) such that our new
  algorithm has {\em constant} regret unlike the previous approaches.
\end{abstract}

\section{Introduction}

We are interested in the online contextual bandit setting, where on
each round we first see a context $x \in \X$, based on which we
choose an action $a \in \A$,
and then observe a reward $r$.
This formalizes several natural scenarios. For example, a common task
at major internet engines is to display the best ad from a pool of
options given some context such as information about the user, the
page visited, the search query issued etc. The action set consists of
the candidate ads and the reward is typically binary based on whether
the user clicked the displayed ad or not. Another natural application
is the design of clinical trials in the medical domain. In this case,
the actions are the treatment options being compared, the context is
the patient's medical record and reward is based on whether the
recommended treatment is a success or not.

Our goal in this setting is to compete with a particular
set of \emph{policies}, which are deterministic rules specifying
which action to choose in each context.
We note that this setting includes as special cases the classical
$K$-armed bandit problem~\citep{LaiRo85} and associative reinforcement
learning with linear reward functions~\citep{Auer2003,
  ChuLiReSc11}.

The performance of algorithms in this setting is typically measured by
the \emph{regret}, which is the difference between the cumulative
reward of the best policy and the algorithm.  For the setting with an
arbitrary set of policies, the achieved regret guarantee is
$\order(\sqrt{KT\ln (N/\delta)})$ where $K$ is the number of actions,
$T$ is the number of rounds, $N$ is the number of policies and
$\delta$ is the probability of failing to achieve the
regret~\citep{BeygelzimerLaLiReSc2010, DudikHsKaKaLaReZh2011}. While
this bound has a desirably small dependence on the parameters $T, N$,
the scaling with respect to $K$ is often too big to be meaningful. For
instance, the number of ads under consideration can be huge, and a
rapid scaling with the number of alternatives in a clinical trial is
clearly undesirable. Unfortunately, the dependence on $K$ is
unavoidable as proved by existing lower
bounds~\citep{auer03multiarmed}.

Large literature on ``linear bandits'' manages to avoid this
dependence on $K$ by making additional assumptions.  For example,
\citet{Auer2003} and \citet{ChuLiReSc11} consider the setting where
the context $x$ consists of feature vectors $x_a\in\R^d$ describing
each action, and the expected reward function (given a context $x$ and
action $a$) has the form $w^Tx_a$ for some fixed vector
$w\in\R^d$. \citet{dani08geometric} consider a continuous action space
with $a\in\R^d$, without contexts, with a linear expected reward
$w^Ta$, which is generalized by \citet{FillippiCaGaSz2010} to
$\sigma(w^T a)$ with a known Lipschitz-continuous link function
$\sigma$. A striking aspect of the linear and generalized linear
setting is that while the regret grows rapidly with the dimension $d$,
it grows either only gently with the number of actions $K$
(poly-logarithmic for~\citealp{Auer2003}), or is independent of $K$
\citep{dani08geometric,FillippiCaGaSz2010}. In this paper, we
investigate whether a weaker dependence on the number of actions is
possible in more general settings. Specifically, we omit the linearity
assumption while keeping the ``realizability''---i.e., we still assume
that the expected reward can be perfectly modeled, but do not require
this to be a linear or a generalized linear model.

We consider an arbitrary class $F$ of functions $f:(\X,\A)\to [0, 1]$
that map a context and an action to a real number.  We interpret
$f(x,a)$ as a predicted expected reward of the action $a$ on context
$x$ and refer to functions in $F$ as \emph{regressors}. For example,
in display advertising, the context is a vector of features derived
from the text and metadata of the webpage and information about the
user. The action corresponds to the ad, also described by a set of
features. Additional features might be used to model interaction
between the ad and the context.  A typical regressor for this problem
is a generalized linear model with a logistic link, modeling the
probability of a click.

The set of regressors $F$ induces a natural set of policies $\Pi_F$
containing maps $\pi_f:\X\to\A$ defined as $\pi_f(x) = \argmax_a
f(x,a)$.  We make the assumption that the expected reward for a
context $x$ and action $a$ equals $f^*(x,a)$ for some unknown function
$f^* \in F$. The question we address in this paper is: Does this
realizability assumption allow us to learn faster?

We show that for an arbitrary function class, the answer to the above
question is ``no''. The $\sqrt{K}$ dependence in regret is in general
unavoidable even with the realizability assumption. Thus, the
structure of linearity or controlled non-linearity was quite important
in the past works.

Given this answer, a natural question is whether it is at least
possible to do better in various special cases.  To answer this, we
create a new natural algorithm, Regressor Elimination (RE), which takes
advantage of realizability. Structurally, the algorithm is similar to
Policy Elimination (PE) of \citet{DudikHsKaKaLaReZh2011}, designed
for the agnostic case (i.e, the general case without realizability
assumption). While PE
proceeds by eliminating poorly performing policies, RE proceeds by
eliminating poorly predicting regressors. However, realizability assumption
allows much more aggressive elimination strategy, different from
the strategy used in PE. The analysis of
this elimination strategy is the key technical contribution of this paper.

The general regret guarantee for Regressor Elimination is
$\order(\sqrt{KT\ln (NT/\delta)})$, similar to the
agnostic case.
However, we also show that for \emph{all} sets of policies
$\Pi$ there exists a set of regressors $F$ such that $\Pi = \Pi_F$ and
the regret of Regressor Elimination is $\order(\ln
(N/\delta))$, i.e., independent of the number of rounds and
actions. At the first sight, this seems to contradict our worst-case
lower bound. This apparent paradox is due to the fact that the
same set of policies can be generated by two very different sets of
regressors. Some regressor sets allow better discrimination of the
true reward function, whereas some regressor sets will lead to the
worst-case guarantee.

The remainder of the paper is organized as follows. In the next
section we formalize our setting and
assumptions. Section~\ref{sec:algo} provides our algorithm which is
analyzed in Section~\ref{sec:proof}. In Section~\ref{sec:lb} we
present the worst-case lower bound, and in Section~\ref{sec:nontriv}, we
show an improved dependence on $K$ in
favorable cases. Our algorithm assumes the exact knowledge of
the distribution over contexts (but not over rewards). In Section~\ref{sec:history} we sketch
how this assumption can be removed. Another major assumption
is the finiteness of the set of regressors $F$. This assumption is
more difficult to remove, as we discuss in Section~\ref{sec:conclusion}.

\section{Problem Setup}

We assume that the interaction between the learner and nature happens
over $T$ rounds. At each round $t$, nature picks a context $x_t \in
\X$ and a reward function $r_t: \A
\to [0, 1]$ sampled i.i.d.\ in each round, according to a fixed
distribution
$D(x,r)$. We assume that $D(x)$ is known (this assumption
is removed in Section~\ref{sec:history}), but $D(r\vert x)$ is unknown.
The learner observes $x_t$, picks an action $a_t\in\A$, and observes the reward for
the action $r_t(a_t)$. We are given a function class
$F:\X\times\A\to [0,1]$ with $|F| = N$, where $|F|$ is the
cardinality of $F$. We assume that $F$ contains a perfect predictor
of the expected reward:

\begin{assumption}[Realizability]
 \label{ass:realizable}
  There exists a function $f^* \in F$ such that $\E_{r\vert x}[r(a)]
  = f^*(x, a)$ for all $x\in\X$, $a\in\A$.
\end{assumption}

We recall as before that the regressor class $F$ induces the policy class
$\Pi_F$ containing maps $\pi_f:\X\to\A$ defined by $f\in F$ as
$\pi_f(x) = \argmax_a f(x,a)$. The performance of an
algorithm is measured by its \emph{expected regret} relative to the
best fixed policy:
\begin{equation*}
  \regret_T = \sup_{\pi_f \in \Pi_F}\sum_{t=1}^T \BigBracks{f^*\bigParens{x_t,\pi_f(x_t)} - f^*(x_t,a_t)}
\enspace.
\end{equation*}
By definition of $\pi_f$, this is equivalent to
\begin{equation*}
  \regret_T = \sum_{t=1}^T \BigBracks{f^*\bigParens{x_t,\pi_{f^*}(x_t)} - f^*(x_t,a_t)}
\enspace. 
\end{equation*}

\section{Algorithm}
\label{sec:algo}

Our algorithm, Regressor Elimination, maintains a set of
regressors that
accurately predict the observed rewards.
In each round, it chooses an action that sufficiently explores
among the actions represented in the current set of regressors (Steps~1--2).
After observing the reward (Step~3), the inaccurate regressors are
eliminated (Step~4).

Sufficient exploration is achieved by solving the convex optimization
problem in Step~1. We construct a distribution $P_t$ over current
regressors, and then act by first sampling a regressor $f\sim P_t$ and
then choosing an action according to $\pi_f$. Similarly to the Policy Elimination algorithm
of~\citet{DudikHsKaKaLaReZh2011}, we seek a distribution $P_t$
such that the inverse probability of choosing an action that agrees
with  \emph{any} policy in the current set is in expectation bounded
from above. Informally, this
guarantees that actions of any of the current policies are
chosen with sufficient probabilities. Using this construction we
relate the accuracy of regressors to the regret of the algorithm
(Lemma~\ref{lem:transfer}).

A priori, it is not clear whether the constraint~\eqref{eq:constraint} is
even feasible. We prove feasibility by a similar argument as
in~\citet{DudikHsKaKaLaReZh2011} (see Lemma~\ref{lem:feasibility} in
Appendix~\ref{app:feasibility}). Compared
with~\citet{DudikHsKaKaLaReZh2011} we are able to obtain
tighter constraints by doing a more careful analysis.

Our elimination step (Step~4) is significantly tighter than a similar
step in~\citet{DudikHsKaKaLaReZh2011}: we eliminate regressors
according to a very strict $\order(1/t)$ bound on the suboptimality of
the least squares error. Under the realizability assumption, this
stringent constraint will not discard the optimal regressor
accidentally, as we show in the next section.  This is the key novel
technical contribution of this work.

 Replacing $D(x)$ in the Regressor Elimination algorithm with the
 empirical distribution over observed contexts is straightforward, as
 was done in ~\citet{DudikHsKaKaLaReZh2011}, and is discussed further
 in Section~\ref{sec:history}.

\begin{algorithm}[!h]

\caption{Regressor Elimination}
\label{alg:reg-elem}

\textbf{Input:}\\
a set of reward predictors $F=\set{f:(\X,\A)\to [0,1]}$\\
distribution $D$ over contexts, confidence parameter $\delta$.\\[4pt]
\textbf{Notation:}\\
$\pi_f(x) :=\argmax_{a'}f(x,a')$.\\[2pt]
$\Rhat_t(f) := \frac1t\sum_{t'=1}^t (f(x_{t'},a_{t'})-r_{t'}(a_{t'}))^2$.\\[2pt]
For $F' \subseteq F$, define\\
\hphantom{---}$A(F',x) := \Set{a\in\A:\:\pi_f(x)=a \text{ for some }f \in F'}$\\[2pt]
$\mu := \min\{1/2K, 1/\sqrt{T}\}$.\\[2pt]
For a distribution $P$ on $F' \subseteq F$, define conditional distribution $P'(\cdot | x)$ on $\A$
as:\\
\hphantom{---}
 w.p.\ $(1 - \mu)$, sample $f \sim P$ and return $\pi_f(x)$, and \\
\hphantom{---}
 w.p.\ $\mu$, return a uniform random $a \in A(F', x)$.\\[2pt]
$\delta_t = \delta / 2Nt^3\log_2(t)$, for $t = 1, 2, \ldots, T$.\\[4pt]
\textbf{Algorithm:}\\[2pt]
$F_0\gets F$\\[2pt]
For $t=1,2,\dots,T$:
\begin{enumerate}
\item Find distribution $P_t$ on $F_{t-1}$ such that
\begin{equation} \label{eq:constraint}
  \!\!\!\!
  \!\!\!\!
  \!\!\!\!
 \forall f\in F_{t-1}:\: \E_{x}\!\Bracks{\!
    \frac{1}{P_t'(\pi_f(x) | x)}\!
   }\! \leq \E_{x}\bigBracks{\Card{A(F_{t-1}, x)}} 
\end{equation}
   
\item Observe $x_t$ and sample action $a_t$ from $P_t'(\cdot | x_t)$.
\item Observe $r_t(a_t)$.
\item Set
\[\!\!\!\!
  \!\!\!\!
  F_t \! =\!\Set{f\in F_{t-1}:\: \Rhat_t(f)<\!\!\min_{f'\in F_{t-1}}\!\!\!\!\Rhat_t(f') +
    \frac{18\ln(1/\delta_t)}{t}}\]
\end{enumerate}
\end{algorithm}

\section{Regret Analysis}
\label{sec:proof}
\def\Ybar{\bar{Y}}

Here we prove an upper bound on the regret of Regressor Elimination. 
The proved bound
is no better than the one for existing agnostic algorithms. This
is necessary,
as we will see in Section~\ref{sec:lb}, where we prove a matching lower bound.

\begin{theorem}
  For all sets of regressors $F$ with $|F|=N$ and all distributions
  $D(x,r)$, with probability $1 - \delta $, the regret of
  Regressor Elimination is $\order(\sqrt{KT\ln(NT/\delta)})$.
  \label{thm:regret}
\end{theorem}
\begin{proof}
    By Lemma~\ref{lem:instantaneous-regret} (proved below), in round $t$ if we sample an action by sampling $f$ from $P_t$ and choosing $\pi_f(x_t)$, then the expected regret is $O(\sqrt{K\ln(NT/\delta)/t})$ with probability at least $1 - \delta/2t^2$. The excess regret for sampling a uniform random action is at most $\mu \leq \frac{1}{\sqrt{T}}$ per round. Summing up over all the $T$ rounds and taking a union bound, the total expected regret is $\order\bigParens{\sqrt{KT\ln(NT/\delta)}}$ with probability at least $1 - \delta$.
  Further, the net regret
  is a martingale; hence the Azuma-Hoeffding inequality with range
  $[0,1]$ applies.  So with probability at least $1-\delta$ we have a regret of
  $\order\bigParens{\sqrt{KT\ln(NT/\delta)} + \sqrt{T \ln (1/\delta)}}
  = \order\bigParens{\sqrt{KT\ln(NT/\delta)}} $.
\end{proof}

\begin{lemma} 
  \label{lem:instantaneous-regret} 
  With probability at least $1 - \delta_t N t \log_2(t) \geq 1 - \delta/2t^2$, we have:
  \begin{enumerate}
  \item $f^* \in F_t$.
  \item For any $f \in F_t$,
  \[ \E_{x, r}[r(\pi_f(x))  -
    r(\pi_{f^*}(x))]\ \leq\ \sqrt{\frac{200K\ln(1/\delta_t)}{t}}. \] 
  \end{enumerate}
\end{lemma}
\begin{proof}
  Fix an arbitrary function $f \in F$. For every round $t$,
  define the random variable
  \[ Y_t = (f(x_t, a_t) - r_t(a_t))^2 - (f^*(x_t, a_t) - r_t(a_t))^2. \]
  Here, $x_t$ is drawn from the unknown data distribution $D$,
  $r_t$ is drawn from the reward distribution conditioned on
  $x_t$, and $a_t$ is drawn from $P_t'$ (which is defined conditioned
  on the choice of $x_t$ and is independent of $r_t$). Note that
  this random variable is well-defined for {\em all} functions $f \in
  F$, not just the ones in $F_t$.

  Let $\Et[\cdot]$ and $\vart[\cdot]$ denote the expectation and
  variance conditioned on all the randomness up to round $t$. Using a form of
  Freedman's inequality from~\cite{BartlettDaHaKaRaTe08} (see Lemma~\ref{lem:freedman}) and noting that $Y_t \leq 1$, we get that with
  probability at least $1 - \delta_t \log_2(t)$, we have
  \begin{align*} 
    &\sum_{t'=1}^t\Etp[Y_{t'}] - \sum_{t'=1}^t
    Y_{t'}\ \\ &\leq 4\sqrt{\sum_{t'=1}^t \vartp[Y_{t'}]\ln(1/\delta_t)}
    + 2\ln(1/\delta_t).  
  \end{align*}

  From Lemma~\ref{lem:E-V}, we see that $\vartp[Y_{t'}] \leq
  4\Etp[Y_{t'}]$ so 
  \begin{align*} 
    &\sum_{t'=1}^t\Etp[Y_{t'}] - \sum_{t'=1}^t
    Y_{t'}\ \\ &\leq 8\sqrt{\sum_{t'=1}^t \Etp[Y_{t'}]\ln(1/\delta_t)}
    + 2\ln(1/\delta_t).  
  \end{align*}
  For notational convenience, define $X =
  \sqrt{\sum_{t'=1}^t\Etp[Y_{t'}]}$, $Z = \sum_{t'=1}^t Y_{t'}$, and $C =
  \sqrt{\ln(1/\delta_t)}$. The above inequality is equivalent to:
    \[X^2 - Z \leq\ 8CX + 2C^2
    \Leftrightarrow (X - 4C)^2 - Z\ \leq\ 18C^2.\]

  This gives $-Z \leq 18C^2$. Since $Z = t(\Rhat_t(f) - \Rhat_t(f^*))$, we get that 
  \[\Rhat_t(f^*)\ \leq\ \Rhat_t(f) + \frac{18C^2}{t}.\]
  By a union bound, with probability at least $1 - \delta_t Nt \log_2(t)$, for all
  $f \in F$ and all rounds $t' \leq t$, we have 
  \[\Rhat_{t'}(f^*)\ \leq\ \Rhat_{t'}(f) + \frac{18\ln(1/\delta_t)}{t'}\]
  and so $f^*$ is not eliminated in any elimination step and remains in $F_t$. 

  Furthermore, suppose $f$ is also not eliminated and survives in
  $F_t$. Then we must have $\Rhat_t(f) - \Rhat_t(f^*) \leq
  18C^2/t$, or in other words, $Z \leq 18C^2$. Thus, $(X - 4C)^2
  \leq 36C^2$, which implies that $X^2 \leq 100C^2$, and hence:
 \begin{equation} \label{eq:regress-regret}
    \sum_{t'=1}^t \Etp[Y_{t'}]\ \leq\ 100\ln(1/\delta_t).
  \end{equation} 
  By Lemma~\ref{lem:transfer} and since $P_t$ is measurable with
  respect to the past sigma field up to time $t-1$, for all $t' \leq
  t$ we have
  \[ \E_{x, r}[r(\pi_f(x))  -
    r(\pi_{f^*}(x))]^2\ \leq\ 2K\Etp_{x_{t'}, r_{t'},
    a_{t'}}[Y_{t'}]. \] 
  Summing up over all $t' \leq t$, and using (\ref{eq:regress-regret})
  along with Jensen's inequality we get that
\begin{equation}
\tag*{\qed}
  \E_{x, r}[r(\pi_f(x))  -
  r(\pi_{f^*}(x))]\ \leq\ \sqrt{\frac{200K\ln(1/\delta_t)}{t}}
  \enspace.
\end{equation}
\renewcommand{\qed}{}
\end{proof}

\begin{lemma} 
  \label{lem:E-V} 
  Fix a function $f \in F$. Suppose we sample $x, r$ from the
  data distribution $D$, and an action $a$ from an arbitrary distribution
  such that $r$ and $a$ are conditionally independent given
  $x$. Define the random variable
 \[ Y = (f(x, a) - r(a))^2 - (f^*(x, a) - r(a))^2. \]
 Then we have
  \[ \E_{x, r, a}[Y] = \E_{x, a}\Bracks{(f(x, a) - f^*(x, a))^2} \]
  \[ \var_{x, r, a}[Y] \leq 4\E_{x, r, a}[Y]. \]
\end{lemma} 

\begin{proof}
   Using shorthands $f_{xa}$ for $f(x,a)$ and $r_a$ for $r(a)$, we
   can rearrange the definition of $Y$ as
\begin{equation}
\label{eq:Y}
   Y = (f_{xa} - f^*_{xa})(f_{xa} + f^*_{xa} - 2r_a)
\enspace.
\end{equation}
  Hence, we have
 \begin{align*}
    \E_{x, r, a} [Y]
    & = \E_{x, r, a}\Bracks{(f_{xa} - f^*_{xa})(f_{xa} + f^*_{xa} - 2r_a)}\\
    &=  \E_{x, a} \E_{r|x}\Bracks{(f_{xa} - f^*_{xa})(f_{xa} + f^*_{xa} - 2r_a)}\\
    &=  \E_{x, a} \Bracks{(f_{xa} - f^*_{xa})\BigParens{f_{xa} + f^*_{xa} - 2\E_{r|x}[r_a]}}\\
    &=  \E_{x, a} \Bracks{(f_{xa} - f^*_{xa})^2}
\enspace,
 \end{align*}
 proving the first part of the lemma. From \eqref{eq:Y}, noting that
 $f_{xa}, f^*_{xa}, r_a$ are between $0$ and $1$, we obtain
 \begin{align*}
    Y^2 &\leq (f_{xa} - f^*_{xa})^2(f_{xa} + f^*_{xa} - 2r_a)^2\\
    &\leq 4(f_{xa} - f^*_{xa})^2
\enspace,
  \end{align*}
yielding the second part of the lemma:
\begin{align}
\notag
   \var_{x, r, a} [Y]  \leq \E_{x, r, a} [Y^2]  &\leq 4\E_{x, r, a}\Bracks{(f_{xa} - f^*_{xa})^2}\\  
\tag*{\qed}
    &= 4\E_{x, r, a} [Y]
\enspace.
  \end{align}
\renewcommand{\qed}{}
\end{proof}

Next we show how the random variable $Y$ defined in
Lemma~\ref{lem:E-V} relates to the regret in a single round:
\begin{lemma} 
\label{lem:transfer} 
In the setup of Lemma~\ref{lem:E-V}, assume further that the action
$a$ is sampled from a conditional distribution $p(\cdot | x)$ which
satisfies the following constraint, for $f' = f$ and $f' = f^*$:
\begin{equation}
  \label{eq:dist-constraint}
  \E_x \left[ \frac{1}{p(\pi_{f'}(x) | x)}\right] \leq K.
\end{equation}
 Then we have
  \[ \E_{x, r}\BigBracks{r\bigParens{\pi_{f^*}(x)} - r\bigParens{\pi_f(x)}}^2
  \ \leq\ 2K\E_{x, r, a}[Y]. \]
\end{lemma}

This lemma is essentially a refined form of theorem 6.1
in~\citet{BeygelzimerLa09} which analyzes the regression approach to
learning in contextual bandit settings.

\begin{proof} 
Throughout, we continue using the shorthand $f_{xa}$
for $f(x,a)$.
Given a context $x$, let $\atil = \pi_f(x)$ and $a^* =
\pi_{f^*}(x)$.
Define the random variable
  \[\Delta_x = \E_{r\vert x}
     \BigBracks{r\bigParens{\pi_{f^*}(x)} - r\bigParens{\pi_f(x)}}
     = f^*_{xa^*} - f^*_{x\atil}
\enspace.
  \]
  Note that $\Delta_x\ge 0$ because $f^*$ prefers $a^*$ over $\atil$ for context~$x$. Also we have
  $f_{x\atil} \geq f_{xa^*}$ since $f$ prefers $\atil$ over $a^*$ for
  context $x$. Thus,
  \begin{equation}
    \label{eq:gap}
    f_{x\atil} - f^*_{x\atil} + f^*_{xa^*} - f_{xa^*}\ \geq\ \Delta_x
\enspace.
  \end{equation}
  As in proof of Lemma~\ref{lem:E-V},
 \begin{align}
\notag
    \E_{r, a\vert x} [Y] &= \E_{a |x }\Bracks{(f_{xa} - f^*_{xa})^2}\\
\notag
     &\geq p(\atil | x)(f_{x\atil}\! -\! f^*_{x\atil})^2 \! +\! p(a^*
     | x)(f^*_{xa^*}\! -\! f_{xa^*})^2\\
\label{eq:cs1}
    &\geq\ \frac{p(\atil|x)p(a^*|x)}{p(\atil|x) + p(a^*|x)}\Delta_x^2
\enspace.
  \end{align}
  The last inequality follows by first applying
the chain
\[\aul \xul^2 + \bul \yul^2 =
   \frac{
      \aul\bul (\xul + \yul)^2
      +
      (\aul\xul - \bul\yul)^2}
    {\aul + \bul}
   \geq
    \frac{\aul\bul}{\aul + \bul}(\xul + \yul)^2
\]
(valid for $\aul,\bul>0$), and then applying
inequality~(\ref{eq:gap}).

For convenience, define
\[
   Q_x = \frac{p(\atil|x)p(a^*|x)}{p(\atil|x)
    + p(a^*|x)},
\enspace
\text{i.e.,}
\enspace
\frac{1}{Q_x} = \frac{1}{p(\atil|x)} +
  \frac{1}{p(a^*|x)}.
\]
Now, since $p$ satisfies the
  constraint~(\ref{eq:dist-constraint}) for $f' = f$ and $f' = f^*$,
  we conclude that
  \begin{equation}
  \label{eq:cs2}
   \E_x\Bracks{\frac{1}{Q_x}}
   =\E_x\Bracks{\frac{1}{p(\atil|x)}}
  + \E_x\Bracks{\frac{1}{p(a^*|x)}}
   \leq 2K
\enspace.
  \end{equation}
  We now have
  \begin{align*}
    \E_x [\Delta_x]^2\ &= \E_x\!\Bracks{\frac{1}{\sqrt{Q_x}}\cdot
    \sqrt{Q_x}\Delta_x}^2\\ 
    &\leq \E_x\!\Bracks{\frac{1}{Q_x}}\,\E_x\!\Bracks{Q_x\Delta_x^2}\\
    &\leq 2K\E_{x,r,a}[Y]
\enspace,
 \end{align*}
  where the first inequality follows from the Cauchy-Schwarz
  inequality and the second from the inequalities \eqref{eq:cs1} and
  \eqref{eq:cs2}.
\end{proof}



\section{Lower bound}
\label{sec:lb}

Here we prove a lower bound showing that the realizability assumption
is \emph{not} enough in general to eliminate a dependence on the
number of actions $K$.  The structure of this proof is similar to an
earlier lower bound~\citep{auer03multiarmed} differing in two ways: it
applies to regressors of the sort we consider, and we work $N$, the
number of regressors, into the lower bound.  Since for every policy
there exists a regressor with argmax on that regressor realizing the
policy, this lower bound also applies to policy based algorithms. 

\def\unif{\text{unif}}

\begin{theorem}\label{thm:lb}
  For every $N$ and $K$ such that $\ln N/\ln K \leq T$, and every
  algorithm $\algo$, there
  exists a function class $F$ of cardinality at most $N$ and a distribution
  $D(x,r)$ for which the realizability assumption holds,
  but the expected regret of $\algo$ is
  $\Omega(\sqrt{KT\ln N/\ln K})$.
\end{theorem}

\begin{proof}
  Instead of directly selecting $F$ and $D$ for which the expected
  regret of $\algo$ is $\Omega(\sqrt{KT\ln N/\ln K})$, we create a distribution over instances
  $(F,D)$ and show that the expected regret of $\algo$ is
  $\Omega(\sqrt{KT\ln N/\ln K})$ when the
  expectation is taken also over our choice of the instance. This will
  immediately yield a statement of the theorem, since the algorithm
  must suffer at least this amount of regret on one of the instances.
  
  The proof proceeds via a reduction to the construction used in the
  lower bound of Theorem 5.1 of~\citet{auer03multiarmed}.
  We will use $M$ different contexts for a suitable number $M$. To
  define the regressor class $F$, we begin with the policy class $G$
  consisting of all the $K^M$ mappings of the form $g: \X
  \to \A$, where $\X = \{1, 2, \ldots, M\}$ and $\A = \{1, 2,
  \ldots, K\}$. We require $M$ to be the largest integer such
  that $K^M \leq N$, i.e., $M = \left\lfloor \ln N/\ln
    K\right\rfloor$.
  Each mapping $g\in G$ defines a regressor
  $f_g\in F$ as follows:
  \[ f_g(x, a) = \begin{cases}
    1/2 + \epsilon & \text{ if } a = g(x)\\
    1/2 & \text {otherwise.}
  \end{cases}\]
%
%
%
%
The rewards are generated by picking a function $f \in F$ uniformly at
random at
the beginning. Equivalently, we choose a mapping $g$ that
independently maps each context $x \in \X$ to a random action $a \in
\A$, and set $f = f_g$. In each round $t$, a context $x_t$ is picked
uniformly from $\X$. For any action $a$, a reward $r_t(a)$ is
generated as a $\{0, 1\}$ Bernoulli trial with probability of $1$
being equal to $f(x, a)$.

Now fix a context $x \in \X$. We condition on all of the randomness
of the algorithm $\algo$, the choices of the
contexts $x_t$ for $t = 1, 2, \ldots, T$, and the values of $g(x')$
for $x' \neq x$. Thus the only randomness left is in the choice of
$g(x)$ and the realization of the rewards in each round. Let $\Pr'$
denote the reward distribution where the rewards of any action $a$ for
context $x$ are chosen to be $\{0, 1\}$ uniformly at random (the
rewards for other contexts $x' \neq x$ are still chosen according to
$f(x', a)$, however), and let $\E'$ denote the expectation under
$\Pr'$.

Let $T_x$ be the rounds $t$ where the context $x_t$ is $x$. Now fix an
action $a \in \A$ and let $S_a$ be a random variable denoting the
number of rounds $t \in T_x$ when $\algo$ chooses $a_t = a$.
Note that conditioned on $g(x) = a$, the random variable $S_a$ counts
the number of rounds in $T_x$ that $\algo$ chooses the optimal
action $a$.

  We use a corollary of Lemma A.1 in~\citet{auer03multiarmed}:
  \begin{corollary}[\citealp{auer03multiarmed}]
    Conditioned on the choices of the contexts $x_t$ for $t = 1, 2,
    \ldots, T$, and the values of $g(x')$ for $x' \neq x$, we have
    \[ \E[S_a | g(x) = a]\ \leq\ {\E}'[S_a] + |T_x|\sqrt{2\epsilon^2 {\E}'[S_a]}. \]
  \end{corollary}
  The proof uses the fact that when $g(x) = a$, rewards chosen using
  $\Pr'$ are identical to those from the true distribution except for
  the rounds when $\algo$ chooses the action $a$.
  
  Thus, if $N_x$ is a random variable that counts the number the rounds
  in $T_x$ that $\algo$ chooses the optimal action for $x$ (without
  conditioning on $g(x)$), we have
  \begin{align*}
    \E[N_x]\ &=\ \E_{g(x)}[\E[S_{g(x)}]]\\
    &\leq\ \E_{g(x)}\BigBracks{\;{\E}'[S_{g(x)}] + |T_x|\sqrt{2\epsilon^2 {\E}'[S_{g(x)}]}\;}\\
    &\leq\ \E_{g(x)}\BigBracks{{\E}'[S_{g(x)}]} +
    |T_x|\sqrt{2\epsilon^2 \E_{g(x)}\bigBracks{{\E}'[S_{g(x)}]}}
\enspace,
  \end{align*}
  by Jensen's inequality. Now note that
  \begin{align*}
    \E_{g(x)}\BigBracks{{\E}'[S_{g(x)}]}\ &=\ \E_{g(x)}\left[{\E}'\left[\sum_{t
          \in T_x} \mathbf{1}\{a_t =
        g(x)\}\right]\right]\\ &=\ \sum_{t \in
      T_x}{\E}'[\E_{g(x)}[\mathbf{1}\{a_t = g(x)\}]]\\ &=\ \sum_{t \in
      T_x}{\E}'\left[\frac{1}{K}\right] \ =\ \frac{|T_x|}{K}
\enspace.
  \end{align*}
  The third equality follows because $g(x)$ is independent of the
  choices of the contexts $x_t$ for $t = 1, 2, \ldots, T$, and $g(x')$
  for $x' \neq x$, and its distribution is uniform on $\A$. Thus 
  \[
   \E[N_x]\ \leq\ \frac{|T_x|}{K} + |T_x|\sqrt{2\epsilon^2
     \frac{|T_x|}{K}}
\enspace.
   \]
  Since in the rounds in $T_x \setminus N_x$, the algorithm $\algo$
  suffers an expected regret of $\epsilon$, the expected regret of
  $\algo$ over all the rounds in $T_x$ is at least
  $\Omega\left(\epsilon |T_x| -
  \frac{\epsilon^2}{\sqrt{K}}|T_x|^{3/2}\right)$. Note that this lower
  bound is independent of the choice of $g(x')$ for $x' \neq x$. Thus,
  we can remove the conditioning on $g(x')$ for $x' \neq x$ and
  conclude that only conditioned on the choices of the contexts $x_t$
  for $t = 1, 2, \ldots, T$, the expected regret of the algorithm over
  all the rounds in $T_x$ is at least $\Omega\left(\epsilon |T_x| -
  \frac{\epsilon^2}{\sqrt{K}}|T_x|^{3/2}\right)$. Summing up over all
  $x$, and removing the conditioning on the choices of the contexts
  $x_t$ for $t = 1, 2, \ldots, T$ by taking an expectation, we get the
  following lower bound on the expected regret of $\algo$:
 \[
  \Omega\Parens{\sum_{x \in \X}\Parens{\epsilon \E[|T_x|] -
      \frac{\epsilon^2}{\sqrt{K}}\E[|T_x|^{3/2}]}}
\enspace.
  \]
  Note that $|T_x|$ is distributed as $\text{Binomial}(T,1/M)$. Thus,
  $\E\bigBracks{|T_x|} = T/M$. Furthermore, by Jensen's inequality
  \begin{align*}
    &\E\bigBracks{|T_x|^{3/2}}\
      \leq\ \sqrt{\E\bigBracks{|T_x|^3}}
\\&\quad{}
    =\ \Parens{\frac{T}{M} + \frac{3T(T-1)}{M^2} +
      \frac{T(T-1)(T-2)}{M^3}}^{\!\!1/2}
\\&\quad{}
   \leq\ \frac{\sqrt{5}T^{3/2}}{M^{3/2}},
  \end{align*}
  as long as $M \leq T$. Plugging these bounds in, the lower bound on
  the expected regret becomes
  \[ \Omega\left(\epsilon T  - \frac{\epsilon^2}{\sqrt{KM}}T^{3/2}\right).\]
  Choosing $\epsilon = \Theta\bigParens{\sqrt{KM/T}}$, we get
  that the expected regret of $\algo$ is lower bounded by
\begin{equation}
\tag*{\qed}
  \Omega(\sqrt{KMT})\ =\ \Omega(\sqrt{K T \ln N /\ln K})
\enspace.
\end{equation}
\renewcommand{\qed}{}
\end{proof}

\section{Analysis of nontriviality}
\label{sec:nontriv}

Since the worst-case regret bound of our new algorithm is the same as
for agnostic algorithms, a skeptic could conclude that there is no
power in the realizability assumption.  Here, we show that in some
cases, realizability assumption can be very powerful in reducing
regret.

\begin{theorem}
  For any algorithm $\algo$ working with a set of policies (rather
  than regressors),
  there exists a set of regressors $F$ and a distribution $D$
  satisfying the realizability assumption such that the regret of
  $\algo$ using the set $\Pi_F$ is $\tilde{\Omega}(\sqrt{TK \ln N})$,
  but the expected regret of
  Regressor Elimination using $F$ is at most $\order\bigParens{\ln(N/\delta)}$.
\end{theorem}
\begin{proof}
  Let $F'$ be the set of functions and $D$ the data distribution that
  achieve the lower bound of Theorem~\ref{thm:lb} for the algorithm
  $\algo$. Using Lemma~\ref{lem:regressor-advice} (see below), there
  exists a set of functions $F$ such that $\Pi_F = \Pi_{F'}$ and the
  expected regret of Regressor Elimination using $F$ is at most
  $\order\bigParens{\ln(N/\delta)}$. This set of functions $F$ and
  distribution $D$ satisfy the requirements of the theorem.
\end{proof}

\begin{lemma} \label{lem:regressor-advice}
  For any distribution $D$ and a set of policies $\Pi$ containing
  the optimal policy, there exists a set of functions $F$ satisfying
  the realizability assumption, such that
 $\Pi = \Pi_F$ and the regret of regressor elimination using $F$ is
  at most $\order\bigParens{\ln(N/\delta)}$.
\end{lemma}
\begin{proof}
  The idea is to build a set of functions $F$ such that $\Pi = \Pi_F$,
  and for the optimal policy $\pi^*$ the corresponding function
  $f^*$ exactly gives the expected rewards for each context $x$ and
  $a$, but for any other policy $\pi$ the corresponding function $f$
  gives a terrible estimate, allowing regressor elimination to
  eliminate them quickly.

  The construction is as follows. For $\pi^*$, we define the function
  $f^*$ as $f^*(x,a)=\E_{x, r}[r(a)]$.  By optimality of $\pi^*$,
  $\pi_{f^*}=\pi^*$.  For every other policy $\pi$ we construct an $f$
  such that $\pi = \pi_f$ but for which $f(x,a)$ is a very bad
  estimate of $\E_{x, r}[r(a)]$ for all actions $a$.  Fix $x$ and
  consider two cases: the first is that $\E_{r|x}[r(\pi(x))] > 0.75$
  and the other is that $\E_{r|x} [r(\pi(x))] \leq 0.75$. In the first
  case, we let $f(x,\pi(x)) = 0.51$.  In the second case we let
  $f(x,\pi(x))=1.0$.  Now consider each other action $a'$ in turn. If
  $\E_{r|x} [r(a')] > 0.25$ then we let $f(x,a') = 0$, and if
  $\E_{r|x} [r(a')] \leq 0.25$ we let $f(x,a')=0.5$. 

  The regressor elimination algorithm eliminates regressor with a
  too-large squared loss regret. Now fix any policy $\pi \neq \pi^*$, and
  the corresponding $f$, define, as in the proof of
  Lemma~\ref{lem:instantaneous-regret}, the random variable
  \[
    Y_t = (f(x_t, a_t) - r_t(a_t))^2 - (f^*(x_t, a_t) - r_t(a_t))^2.
  \]
  Note that \begin{equation} \label{eq:constant-gap}
    \Et[Y_{t}] = \E_{x_{t}, a_{t}}[(f(x_{t},a_{t}) - f^*(x_{t},a_{t}))^2]
\geq \frac{1}{20},
  \end{equation}
since for all $(x,a)$, $(f(x, a) - f^*(x, a))^2 \geq \frac{1}{20}$ by
construction.  This shows that the expected regret is significant.


Now suppose $f$ is not eliminated and remains in $F_t$. Then by equation~\ref{eq:regress-regret} we get:
\[ \frac{t}{20}\ \leq\ \sum_{t'=1}^t \Etp[Y_{t'}]\ \leq\ 100\ln (1/\delta_t).\]
The above bound holds with probability $1-\delta_t Nt\log_2(t)$ uniformly for all $f \in F_t$.
Using the choice of $\delta_t = \delta / 2Nt^3\log_2(t)$, we note that the bound fails to hold when $t > 10^6\ln(N/\delta)$. Thus, within $10^6 \ln(N/\delta)$ rounds all suboptimal
regressors are eliminated, and the algorithm suffers no regret
thereafter. Since the rewards are bounded in $[0, 1]$, the total
regret in the first $10^6 \ln(N/\delta)$ rounds can be at most $10^6\ln(N/\delta)$, giving us the desired bound.
\end{proof}

\section{Removing the dependence on $D$} 
\label{sec:history}

While Algorithm~\ref{alg:reg-elem} is conceptually simple and enjoys
nice theoretical guarantees, it has a serious drawback that it depends
on the distribution $D$ from which the contexts $x_t$'s are drawn in
order to specify the constraint~\eqref{eq:constraint}. A similar issue
was faced in the earlier work of~\citet{DudikHsKaKaLaReZh2011}, where
they replace the expectation under $D$ with a sample average over the
contexts observed. We now discuss a similar modification for
Algorithm~\ref{alg:reg-elem} and give a sketch of the regret analysis.

The key change in Algorithm~\ref{alg:reg-elem} is to replace the
constraint~\eqref{eq:constraint} with the sample version. 
Let $H_t = \{x_1, x_2, \ldots, x_{t-1}\}$, and denote by $x \sim H_t$ the act of selecting a context $x$ from $H_t$ uniformly at random.
Now we pick a distribution $P_t$ on $F_{t-1}$ such that 
\begin{multline}
\label{eq:constraint-sample}
\forall f\in F_{t-1}:\: \E_{x \sim H_t}\!\Bracks{\!
    \frac{1}{P_t'(\pi_f(x) | x)}\!
   }\! \leq \E_{x \sim H_t}\bigBracks{\Card{A(F_{t-1}, x)}} 
\end{multline}
Since Lemma~\ref{lem:feasibility} applies to any distribution on the contexts, in particular, the uniform distribution on $H_t$,
this constraint is still feasible. To justify this sample based
approximation, we appeal to Theorem 6 of~\citet{DudikHsKaKaLaReZh2011}
which shows that for any $\epsilon \in (0,1)$ and $t \geq 16K\ln
(8KN/\delta)$, with probability at least $1-\delta$
\begin{align*}
  \E_{x\sim D}&\Bracks{\frac{1}{P_t'(\pi_f(x) | x)}}\\  
  &\leq (1+\epsilon)\E_{x\sim H_t}\Bracks{\frac{1}{P_t'(\pi_f(x) | x)}} + \frac{7500}{\epsilon^3}K.
\end{align*}
Using Equation~\eqref{eq:constraint-sample}, since $\Card{A(F_{t-1},
  x_{t'})} \leq K$, we get 
\begin{align*}
  \E_{x\sim  D}&\Bracks{\frac{1}{P_t'(\pi_f(x) | x)}} \leq 7525K,
\end{align*}
using $\epsilon = 0.999$. The remaining analysis of the algorithm
remains the same as before, except we now apply
Lemma~\ref{lem:transfer} with a worse constant in the
condition~\eqref{eq:dist-constraint}.



\section{Conclusion}
\label{sec:conclusion}

The included results gives us a basic understanding of the realizable
assumption setting: it can, but does not necessarily, improve our
ability to learn.

We did not address computational complexity in this paper.  There are
some reasons to be hopeful however.  Due to the structure of the
realizability assumption, an eliminated regressor continues to have an
increasingly poor regret over time, implying that it may be possible
to avoid the elimination step and simply restrict the set of
regressors we care about when constructing a distribution.  A basic
question then is: can we make the formation of this distribution
computationally tractable?

Another question for future research is the extension to infinite
function classes. One would expect that this just involves replacing
the log cardinality with something like a metric entropy or Rademacher
complexity of $F$. This is not completely immediate since we are
dealing with martingales, and direct application of covering arguments
seems to yield a suboptimal $\order(1/\sqrt{t})$ rate in
Lemma~\ref{lem:instantaneous-regret}. Extending the variance based
bound coming from Freedman's inequality from a single martingale to a
supremum over function classes would need a Talagrand-style
concentration inequality for martingales which is not available in the
literature to the best of our knowledge. Understanding this issue
better is an interesting topic for future work.

\paragraph{Acknowledgements}

This research was done while AA, SK and RES were visiting Yahoo!.

\bibliographystyle{plainnat} 
\bibliography{bib}

\appendix

\section{Feasibility}
\label{app:feasibility}

\begin{lemma}
  There exists a distribution $P_t$ on $F_{t-1}$ satisfying the
  constraint~\eqref{eq:constraint}.
  \label{lem:feasibility}
\end{lemma}


\begin{proof}
  Let $\Delta_{t-1}$ refer to the space of all distributions on
  $F_{t-1}$. We observe that $\Delta_{t-1}$ is a convex, compact
  set. For a distribution $Q \in \Delta_{t-1}$, define the conditional distribution $\tilde{Q}(\cdot | x)$ on $\A$ as sample $f \sim Q$, and return $\pi_f(x)$. Note that $Q'(a | x) = (1 - \mu)\tilde{Q}(a | x) + \mu/K_x$, where $K_x := |A(F_{t-1}, x)|$ for notational convenience.

  The feasibility of constraint~\eqref{eq:constraint}
  can be written as
  \[ \min_{P_t \in \Delta_{t-1}}\max_{f \in F_{t-1}} \E_{x}\Bracks{
    \frac{1}{P_t'(\pi_f(x)|x)}
  } \leq \E_{x}\Bracks{\Card{A(F_{t-1},x)}}.    \]

  The LHS is equal to
  \begin{align*}
    \min_{P_t \in \Delta_{t-1}}\max_{Q \in \Delta_{t-1}} \E_{x}\BiggBracks{\,
      \sum_{f \in F_{t-1}}\frac{Q(f)}{P_t'(\pi_f(x)|x)}},
  \end{align*}
  where we recall that $P_t'$ is the distribution induced on $\A$ by
  $P_t$ as before. The function
  \[
    \E_x\BiggBracks{ \,\sum_{f \in
      F_{t-1}}\frac{Q(f)}{P_t'(\pi_f(x)|x)}
     }
   \]
  is linear (and hence
  concave) in $Q$ and convex in $P_t$. Applying Sion's Minimax Theorem
  (stated below as Theorem~\ref{thm:sion}), we see that the LHS is
  equal to
  \begin{align*}
    &\max_{Q \in \Delta_{t-1}} \min_{P_t \in \Delta_{t-1}}
    \E_{x}\BiggBracks{\, \sum_{f \in F_{t-1}}\frac{Q(f)}{P_t'(\pi_f(x)|x)} } \\ 
    &\leq \max_{Q \in \Delta_{t-1}}
    \E_{x}\BiggBracks{\,\sum_{f \in F_{t-1}}\frac{Q(f)}{Q'(\pi_f(x)|x)} } \\ 
    &= \max_{Q \in \Delta_{t-1}}
    \E_{x}\BiggBracks{\, \sum_{a \in A(F_{t-1}, x)} \sum_{f \in F_{t-1}: \pi_f(x) = a}\frac{Q(f)}{Q'(a|x)} } \\
    &= \max_{Q \in \Delta_{t-1}}
    \E_{x}\BiggBracks{\, \sum_{a \in A(F_{t-1}, x)} \frac{\tilde{Q}(a|x)}{Q'(a|x)} }\\  
    &= \max_{Q \in \Delta_{t-1}}
    \E_{x}\BiggBracks{\, \frac{1}{1-\mu}\cdot\sum_{a \in A(F_{t-1}, x)} \BiggBracks{\, 1 - \frac{\mu}{K_xQ'(a | x)}}}\\  
    &\leq \max_{Q \in
      \Delta_{t-1}} \E_{x}\bigBracks{K_x}
\enspace.
  \end{align*}
  The last inequality uses the fact that for any distribution $P$ on $\{1, 2, \ldots, K\}$, $\sum_{i=1}^K [1/P(i)]$ is minimized when all $P(i)$ equal $1/K$.
  Hence the constraint is always feasible. 
\end{proof}

\begin{theorem}[see Theorem 3.4 of \citealp{Sion58}]
\label{thm:sion}
Let $ U$ and $ V$ be compact and convex sets, and
$\phi: U\times V\to\R$ a function which for all $v\in V$ is convex and continuous in $u$ and for all $u\in U$ is concave
and continuous in $v$. Then
\[
    \min_{u\in U} \max_{v\in V} \phi(u,v)
=
    \max_{v\in V} \min_{u\in U} \phi(u,v)
\enspace.
\]
\end{theorem}

\section{Freedman-style Inequality}
 
\begin{lemma}[see \citealp{BartlettDaHaKaRaTe08}] \label{lem:freedman}
  Suppose $X_1, X_2, \ldots, X_T$ is a martingale difference sequence with $|X_t| \leq b$ for all $t$. Let $V = \sum_{t=1}^T \var_t[X_t]$ be the sum of conditional variances. Then for any $\delta < 1/e^2$, with probability at least $1 - \log_2(T)\delta$ we have
  \[ \sum_{t=1}^T X_t\ \leq\ 4\sqrt{V \ln(1/\delta)} + 2b\ln(1/\delta).\]
\end{lemma}

\end{document}